\documentclass{article} % For LaTeX2e
\usepackage{nips12submit_e,times}

\usepackage{color}
\definecolor{darkblue}{rgb}{0, 0.08, 0.45}

\usepackage{amssymb,amsmath,mathrsfs,amsthm,graphicx,url,hyperref}
\usepackage[square,numbers]{natbib}

\theoremstyle{plain}
\newtheorem{proposition}{Proposition}

\theoremstyle{definition}
\newtheorem{definition}{Definition}

\title{Anomaly Classification with the Anti-Profile Support Vector Machine}

\author{
Wikum Dinalankara\\
Center for Bioinformatics and Computational Biology\\
Department of Computer Science\\
University of Maryland\\
College Park, MD 20742\\
\texttt{wikum@cs.umd.edu} \\
\And
H\'{e}ctor Corrada Bravo\\%\thanks{\url{http://cbcb.umd.edu/~hcorrada}}\\
Center for Bioinformatics and Computational Biology\\
Department of Computer Science\\
University of Maryland\\
College Park, MD 20742\\
\texttt{hcorrada@umiacs.umd.edu}
}

\nipsfinalcopy % Uncomment for camera-ready version

\begin{document}
\bibliographystyle{plainnat}

\maketitle

\begin{abstract}
We introduce the anti-profile Support Vector Machine (apSVM) as a novel
algorithm to address the anomaly classification problem, an extension
of anomaly detection where the goal is to distinguish data samples
from a number of anomalous and heterogeneous classes based on their
pattern of deviation from a normal stable class. 
We show
that under heterogeneity assumptions defined here that the apSVM can
be solved as the dual of a standard SVM with an indirect kernel that measures similarity
of anomalous samples through similarity to the stable normal
class. We characterize this indirect kernel as the
inner product in a Reproducing Kernel Hilbert Space between
representers that are projected to the subspace spanned by the
representers of the normal samples. We show by simulation and
application to cancer genomics datasets that the anti-profile SVM
produces classifiers that are more accurate and stable
than the standard SVM in the anomaly classification setting.

  %Support Vector Machines can be seen as machine learning tools that obtain
  %optimal decision boundaries  by considering similarities between it's support
 %vectors. We demonstrate how to 
\end{abstract}

\section{Introduction}

The task of anomaly, or outlier, detection~\cite{Steinwart:2005:CFA:1046920.1058109,Manevitz,Chandola} is to identify data samples
that deviate significantly from a class for which training samples are available. We explore anomaly classification as an extension to this setting, where the goal is to distinguish data samples from a number of anomalous and heterogeneous classes based on their pattern of deviation from a normal stable class. Specifically, presented with samples from a normal class, along with samples from 2 or more anomalous classes, we want to train a classifier to distinguish samples from the anomalous classes. Since the anomalous classes are heterogeneous using deviation from the normal class as the basis of classification instead of building a classifier for the anomalous classes that ignores samples from the normal class may lead to classifiers and results that are more stable and reproducible.

The motivation for exploring this learning setting is from recent
results in cancer genomics~\cite{Hansen:2011gu}. In particular, it was
shown that hyper-variability in certain genomic measurements (DNA
methylation and gene expression) in specific regions is a stable
cancer mark across many tumor types. Furthermore, this
hyper-variability increases during stages of cancer progression. This
led us to the question of how to distinguish samples from different
stages in the presence of hyper-variability. In essence, how to
distinguish samples from different anomalous classes (given by cancer progression stage) based on deviation from a well-defined normal class (measurements from non-cancerous samples).

We introduce the anti-profile Support Vector Machine (apSVM) as a novel algorithm suitable for the anomaly classification task. It is based on the idea of only using the stable normal class to define basis functions over which the classifier is defined. We show that the dual of the apSVM optimization problem is the same as the dual of the standard SVM with a modified kernel function. We then show that this modified kernel function has general properties that ensure better stability than the standard SVM in the anomaly classification task.

The paper is organized as follows: we first present the anomaly classification setting in detail; we next describe the Anti-Profile Support Vector Machine (apSVM), and show that the dual of the optimization problem defined by it is equivalent to the dual of the standard SVM with a specific kernel modification; we next show that this kernel modification leads directly to a theoretical statement of the stability of the apSVM compared to the standard SVM in the anomaly classification setting;
we next show simulation results describing the performance and
stability of the apSVM; and finally, we present results from cancer
genomics showing the benefit of approaching classification problems in
this area from the anomaly classification point of view.

\section{The anomaly classification problem}

We present the anomaly classification problem in the binary case, with
two anomalous classes. Assume we are given training samples in
$\mathbb{R}^p$ from three classes: $m$ datapoints from normal class
$Z$, and $n$ training datapoints as pairs $\langle
x_1,y_1\rangle,\ldots…, \langle x_n,y_n \rangle$ with labels
$y_i \in \{-1,1\}$ indicating membership of $x_i$ in one of two
anomalous classes $A^-$ and $A^+$. Furthermore, we assume that the
anomalous classes are heterogeneous with respect to normal class
$Z$. Figure 1a illustrates this learning setting for DNA methylation
data~\cite{Hansen:2011gu} (see \ref{section:bio} for details on this aspect
of cancer epigenetics). It is a two-dimensional embedding (using PCA)
of DNA methylation data for normal colon tissues along with benign
growths (adenomas) and cancerous growths (tumor). Variability in these
specific measurements increases from normal to adenoma to tumor. We would like to build stable and robust classifiers that distinguish benign growths from tumors.

\begin{figure}
\centering
\includegraphics[width=\textwidth]{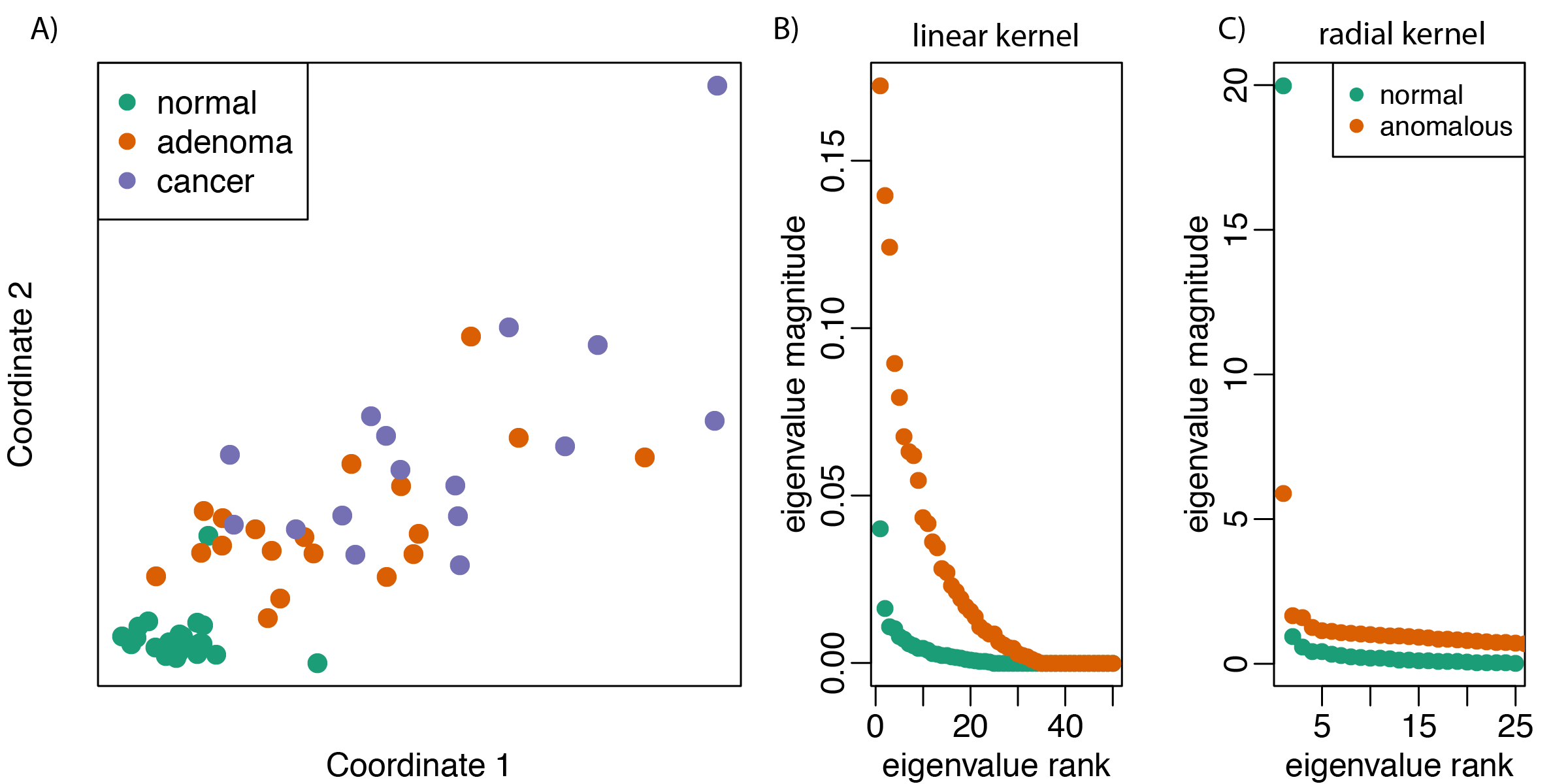}
\caption{(A) Principal component analysis of DNA methylation
  data~\cite{Hansen:2011gu}. Variability in methylation measurements
  increases from normal to benign lesions (adenoma) to malignant
  lesions (cancer). The heterogeneity of the adenoma and cancer
  anomalous samples is the defining feature of the anomaly
  classification setting. (B) and (C) illustration of the
  heterogeneity assumption in Definition~\ref{def:ha}. For both a
linear and radial basis function kernel, the magnitude of eigenvalues
of the kernel matrix is larger for the anomalous classes.}
\end{figure}

Next we seek to formalize the heterogeneity assumption of the anomaly classification problem. Intuitively, the heterogeneity assumption we make is that given random samples of the same size from the normal class and from the anomalous classes, in expectation, the sample covariance of the anomalous samples is always larger than the covariance of the normal samples. We state our assumption in the case of Reproducing Kernel Hilbert Spaces (RKHS) since we will use this machinery throughout the paper~\cite{scholkopf2002lk,wahba1999svm}. Recall that in a Bayesian interpretation of this setting, the kernel function associated with a RKHS serves as the covariance of a Gaussian point process.

\begin{definition}[Heterogeneity Assumption] 
\label{def:ha}
Let $\mathscr{H}$ be a Reproducing Kernel Hilbert Space with
associated kernel function $k$. Let $K^m_Z$ and $K^m_A$ be the kernel
matrices resulting from evaluating the kernel function for a sample of
size $m$ of points in the normal and anomalous classes
respectively. The heterogeneity assumption is that for every integer
$m$, there exists $\epsilon \in \mathbb{R}$, where $0 < \epsilon < 1$ such
that $\mathbf{E} \frac{\det K^m_Z}{\det K^m_A} < \epsilon$.
\end{definition}

Figures 1b and c show that the heterogeneity assumption is satisfied in the
DNA methylation data for both linear and radial basis function
kernels. Each figure shows the magnitude of the eigenvalues of the
resulting kernel matrices. The magnitude of the eigenvalues in both
cases is larger for the anomalous classes.

The heterogeneity assumption gives us a hint to construct classifiers that deal with the heterogeneity of the anomalous classes. In Section~\ref{sec:generalResult} we show that heterogeneity has an impact on robustness and stability of classifiers built from training sets of the anomalous classes. Our goal is to use samples from the stable normal class to create classifiers that are robust. We describe the anti-profile SVM as an extension to Support Vector Machines that accomplishes this goal.

\section{The anti-profile SVM}

Support Vector Machines(SVMs) are one of the primary machine learning tools used for
classification. SVMs operate by learning the maximum-margin separator between
two groups of data provided at training time. Any new observation provided to the
SVM is classified by determining which side of the separator the new observation
lies in. An important advantage of SVMs is that by applying the kernel trick, it
is possible to find a hyperplane in a higher dimensional space where the two
given classes are linearly separable, even when they are not linearly separable
in their original feature spaces, and by virtue of the kernel trick this computation can be 
performed at no significant cost. While primarily designed for binary classification, SVMs have been extended for
many other problems, such as multi-class classification and function estimation.

\subsection{The SVM as weighthed voting of basis functions}
Here we review SVMs from a function approximation perspective~\cite{wahba1999svm}: consider a set of $n$ observations, each observation being drawn from $X \times Y$, where $X \in 
\mathbb{R}^p $, and $Y \in \{-1, 1\}$. Here $p$ is the number of features in each observation, or the 
dimensionality of the feature space. Thus each observation  consists
of a pair $\langle x_i, y_i \rangle$, $x_i \in 
\mathbb{R}^p$ and $y_i \in \{-1, 1\}$, for $i=1..n$; here $y_i$ indicates which of the two classes the observation belongs 
to. If we introduce a new observation $x'$ which needs to be classified, then the classification 
problem amounts to comparing $x'$ to the existing set of points and combining the comparisons to 
make a decision.

To make the comparisons between observations, we make use of a similarity
function. Let $k(x_i, x_j)$ be a positive-definite similarity function which compares two points $x_i, x_j \in \mathbb{R}^p$. Weighing 
the similarity of the new observation to each existing observation, the difference of the sum of 
weighted similarities for the two groups will provide the 
necessary classification: $g(x) = d+ \sum_{i=1}^{n} c_i k(x,
x_i)$. Here $c_i \ge 0 \,\,\,\, \forall i$ is the weight associated with each
point, and $d$ is a bias term. Classification is then based on the
sign of the expansion: $f(x) = sgn \left [ g(x) \right ]$.

%If the similarity function$\psi(., .)$ is a distance measure, then the classifier is a maximum margin 
%separator for the two classes. For any point lying on the separator, the weighted similarities for 
%the two groups are equal, i.e. $g(x) = d$. Since the maximum margin separator can be defined by a 
%subset of the points that lie closest to the separator, for these points the weight $c_i > 0$; these 
%are the support vectors. For all other points, $c_i = 0$ i.e. they contribute zero weight to the 
%separator. Essentially, support vector machines can be viewed as a scheme based
%on weighted votes from the support vectors. 
%
%The given similarity function can be re-defined as a basis function $h(.)$ for each point as follows:
%\[ h_i(x) = \psi(x, x_i) \]
%Hence the classifier is built as a linear combination of the basis functions $h_i(.)$:
%\[ f(x) = sgn\left [ \sum_{i=1}^{n}y_i c_i h_i(x) + d \right ]\]
%

Usually in SVMs function $k$ is further assumed to have the
reproducing property in a Reproducing Kernel Hilbert Space
$\mathcal{H}$ associated with $k$: $\langle f, k(x,\cdot) \rangle_{\mathcal{H}} =
f(x)$ for all $f \in \mathcal{H}$, and in particular $\langle
k(x,\cdot), k(y,\cdot) \rangle_{\mathcal{H}} = k(x,y)$. In this case,
the basis functions in the classifier correspond to representers
$k(x,\cdot)$. In the standard SVM, the representers of all training
points are potentially used as basis functions in the classifier, but
effectively only a small number of representers are used as basis functions, namely the Support Vectors.
However, for a given problem, we may choose a different set of points for the derivation of the set of basis 
functions; the basis functions determine how the similarities are measured for a new point. 
%Choosing a new set of basis functions adds an extra level of complexity to the classifier; for a 
%given new point for classification, instead of it being simply compared to the existing set of 
%observations, it will be compared to the existing observations which in turn will be compared to 
%the set of points defining the basis.
% 
\subsection{The Anti-Profile SVM optimization problem}

The core idea in the anti-rofile SVM (apSVM) is to make use of this
characterization of the Support Vector Machine as a linear expansion
of basis functions defined by representers of training samples. In
order to address the heterogeneity assumption underlying the anomaly
classification problem we define basis functions only using samples
from the stable normal class. 

Formally, we restrict the set of functions available to define the
subspace of $\mathcal{H}$ spanned by the representers of samples
$z_1,\ldots,z_m$ from normal class $Z$: $f(x) = d+\sum_{i=1}^m c_i
k(z_i,x)$. To estimate coefficients $c_i$ in the basis expansion we
apply the usual regularized risk functional based on hinge loss

\[
R_{\mathrm{reg}}(f) = \frac{1}{n} \sum_{j=1}^n (1-y_i f(x_i))_+ + \frac{\lambda}{2}
\|h\|^2_{\mathcal{H}},
\]

where $(\cdot)_+=\max(0,\cdot)$, $f(x)$ is defined as $f(x) = d +
h(x)$, and $\lambda>0$ is a regularization parameter.
By the reproducing kernel property, we have in this
case that $\|h\|^2_{\mathcal{H}}=c'K_nc$ where $K_n$ is the kernel
matrix defined on the $m$ normal samples. 

The minimizer of the empirical risk functional is given by the
solution of a quadratic optimization problem, similar to the standard
SVM, but with two kernel matrices used: $K_n$, defined in the previous
paragraph, and $K_s$, which
contains the evaluation of kernel function $k$ between anomalous
samples $x_1,\ldots,x_n$ and normal samples $z_1,\ldots,z_m$:

\begin{eqnarray}
\label{eq:primal_problem}
\min_{d,c,\xi} & e^T\xi + \frac{n\lambda}{2} c^TK_nc \\
\textrm{s.t.} & Y(K_s c + de) + \xi \geq e, \xi \geq 0 \nonumber
\end{eqnarray}

Here we use slack variables $\xi = (\xi_1, \xi_2, ..., \xi_n)'$, denote
the unit vector of size $n$ as $e$, and define matrix $Y$ as the
diagonal matrix such that $Y_{ii}=y_i$.

\subsection{Solving the apSVM optimization problem}
The Lagrangian of problem~\ref{eq:primal_problem} is given by
\[ 
L(c,d,\xi,\alpha,\beta) = e^T \xi + \frac{n\lambda}{2}c^T K_n c - \alpha^T \left [ Y (K_s
  c + de) + \xi - e\right ] - \beta^T \xi 
\]
where $\alpha_{n \times 1} = (\alpha_1, ..., \alpha_n)^T$ and $\beta_{n \times 1} = (\beta_1, ..., \beta_s)^T$ are the Lagrangian multipliers. Minimizing with respect to $z, c$ and $d$, we find that the Wolfe dual
of problem~\ref{eq:primal_problem} is 
%\[ \frac{\partial L_p}{\partial z} = 0 \Rightarrow e^T - \alpha^T -\mu^T = 0 \]
%\[ \frac{\partial L_p}{\partial d} = 0 \Rightarrow \alpha^T Y = 0 \]
%\[ \frac{\partial L_p}{\partial c} = 0 \Rightarrow \frac{\lambda}{2} 2 c^T K_n - \alpha^TYK_s^T = 0 \]
%\[ c^T = \frac{1}{\lambda}\left ( \alpha^TYK_s^TK_n^{-1} \right ) \]
%\[ c = \frac{1}{\lambda}\left [ (K_n^{-1})^T (YK_s^T)^T \alpha \right ] \]
%\[ c = \frac{1}{\lambda}\left [ K_n^{-1} K_s^T Y \alpha \right ] \]
%Applying these results to $L_p$, we get
%\[ L_d = e^T \alpha + \frac{\lambda}{2}. \frac{\alpha^TYK_s^TK_n^{-1}}{\lambda} . K_n . \frac{K_n^{-1}K_sY\alpha}{\lambda} - \alpha^T Y K_s^T \frac{K_n^{-1}K_sY\alpha}{\lambda} \]
%\[ L_d = e^T \alpha - \frac{\alpha^T Y K_s^T K_n^{-1}K_s Y \alpha}{2 \lambda} \]

\begin{eqnarray}
\label{eq:dual_problem}
\max_{\alpha} & e^T \alpha - \frac{1}{2n\lambda} \alpha^T Y \tilde{K} Y \alpha\\
\textrm{s.t.} & 0 \le \alpha \le e, e^T Y \alpha = 0 \nonumber
\end{eqnarray}

where $\tilde{K} = K_s K_n^{-1}K_s^T$. Here we assume $K_n^{-1}$
represents a pseudo-inverse in the case where $K_n$ is not positive definite.

For a standard SVM, the objective of the Wolfe dual is $e^T \alpha -
\frac{1}{2n\lambda} \alpha^T Y K’ Y \alpha$, with $K$ the kernel matrix the training datapoints.
Thus the dual problem of the apSVM has the same form as the standard
SVM dual problem with the exception that kernel matrix $K$
is replaced by induced kernel matrix $\tilde K$ in the apSVM. Kernel matrix $\tilde K$  
essentially represents an indirect kernel between anomalous samples
induced by the set of basis functions determined by the samples from
the normal class. Since the essential form of the SVM solution is unchanged by the modification,
this provides the additional advantage that the modified SVM can be solved 
by the same tools that solve a regular SVM, but with a different kernel matrix
provided. For our particular problem domain, we use the indirect
kernel to represent deviation from the 
profile of normal samples, and thus refer to this classifier as the anti-profile SVM.

\subsection{Characterizing the indirect kernel}
\label{sec:generalResult}

We saw above that the apSVM can be solved as a standard SVM with induced kernel $\tilde{K} = K_s K_n^{-1}K_s^T$. In this section we characterize this indirect kernel, and state a general result that elucidates how the apSVM can produce classifiers that are more robust and reproducible that a standard SVM in this setting. 

\begin{proposition} Let $P_Z$ be the linear operator that projects
  representers $k(x,.) \in \mathscr{H}$ to the space spanned by the representers of the $m$ normal samples of the anomaly classification problem. Induced kernel $\tilde{k}$ satisfies $\tilde{k}(x,y) = k(P_Z k(x,.), P_Z k(y,.))$.
\end{proposition}

\begin{proof}
Projection $P_Z k(x,.)$ is defined as $P_Z k(x,.) = \sum_{i=1}^m \hat{\beta}_i k(z_i,.)$ where

\begin{eqnarray}
\hat{\beta} & = & \arg \min_{\beta} \frac{1}{2} \|k(x,.) - \sum_i \beta
k(z_i,.) \|^2_\mathscr{H} \nonumber \\
{} & = & \arg \min_{\beta} \frac{1}{2} \langle k(x,.)-\sum_i\beta k(z,.),
k(x,.)-\sum_i\beta k(z_i,.)\rangle_\mathscr{H} \nonumber \\
{} &= & \arg \min_{\beta} \frac{1}{2} \left( \sum_{i,j} \langle k(z_i,.),k(z_j,.)\rangle_\mathscr{H} - \sum_i
\langle k(x,.),k(z_i,.)\rangle_\mathscr{H} \right) \nonumber \\
{} &= & \arg \min_{\beta} \frac{1}{2} \left( \beta^T K_n \beta - k_{zx}^T\beta \right),
\end{eqnarray}

where $k_{zx}$ is the vector with element $i$ equal to
$k(z_i,x)$. From (3) we get $\hat{\beta}=K_n^{-1} k_{zx}$. Therefore
$\langle P_Z k(x,.), P_Z k(y,.)\rangle_\mathscr{H}=k_{zx}^T K_n^{-1} k_{zy}=\tilde{k}(x,y)$.
\end{proof}

This proposition states that the indirect kernel is the inner product
in Reproducing Kernel Hilbert Space $\mathscr{H}$ between the
representers of anomalous samples projected to the space spanned by 
the representers of normal samples. By the heterogeneity assumption of
Definition~\ref{def:ha}, the space spanned by any subset of
anomalous samples will be smaller after the projection. In particular,
the smallest sphere enclosing the projected representers will be
smaller, and from results such as the Vapnik-Chapelle support vector
span rule~\cite{vapnik2000bee}, classifiers built from this projection will be more
robust and stable.

%\subsection{Selection of reference group}

%To suit each application of this method, it is important to choose the reference
%group carefully. The reference group should exhibit the property that, whatever
%differences are present between the two classification classes 
%(with respect to the chosen similarity measurement), they are amplified when
%measured with respect to the reference group. Essentially, we can say that
%\[ \psi(x_{sp}, x_{sq}) > \left | \sum_{i=1}^{n} \psi( x_{sp}, x_{ni} )  - \sum_{i=1}^{n} \psi( x_{sq}, x_{ni} )   \right | \]
%
%In particular, a class of
%observations that show more stability and cluster together well compared to the
%two classification groups would serve well as a reference group.
%
%For our problem of classification between tumor classes cancer and adenoma,
%normals are the obvious choice of reference, since they exhibit greater
%stability and the two tumor classes exhibit difference in variability as
%measured from that of normals.
%

\section{Simulation Study}

We first present simulation results that show that the apSVM obtains
better accuracy in the anomaly classification setting while providing
stable and robust classifiers. We generated samples from three normal distributions as follows: if $A^+$ and 
$A^-$ are the anomalous classes that we need to distinguish, and $Z$
is the normal class, then for a given feature we draw datapoints from distributions 
$Z = N(0, \sigma^2_N), A^- = N(0, \sigma^2_{A^-})$ and $ A^+ =
N(0,\sigma^2_{A^+})$. To simulate our problem setting, we
set $ \sigma^2_Z <\sigma^2_{A^-} < \sigma^2_{A^+} $.

Results have been obtained from tests written on R
(version 2.14) with R packages kernlab (version 0.9-14)~\cite{kernlab} and svmpath 
(version 0.952). The svmpath tool provides a fitting for 
the entire path of the SVM solution to a model at little additional computational cost~\cite{Hastie:2004uj}. Using the 
resulting fit, the SVM classifications for any given cost parameter can be
obtained. For our experiments, the testing set accuracy was computed for each value of 
cost along the regularization path, and the best accuracy possible was obtained; 
ties were broken by considering the option with the least number of support vectors used.
Note that a small ridge parameter (1e-10) was used in the svmpath method to
avoid singularities in computations.

Each training set contained 20 samples from each of $A^-$ and $A^+$ classes, while
each testing set contained 5 samples from each class; 20 samples from
class $Z$ were used for the anti-profile SVM. For a given number of features, each test was run 10 times
and the mean accuracy computed. To estimate the hyperparameter for the radial basis 
kernel, the inverse of the mean distance between 5 normal and 5 anomalous
samples (chosen randomly) was used.

\begin{figure}
\centering
\includegraphics{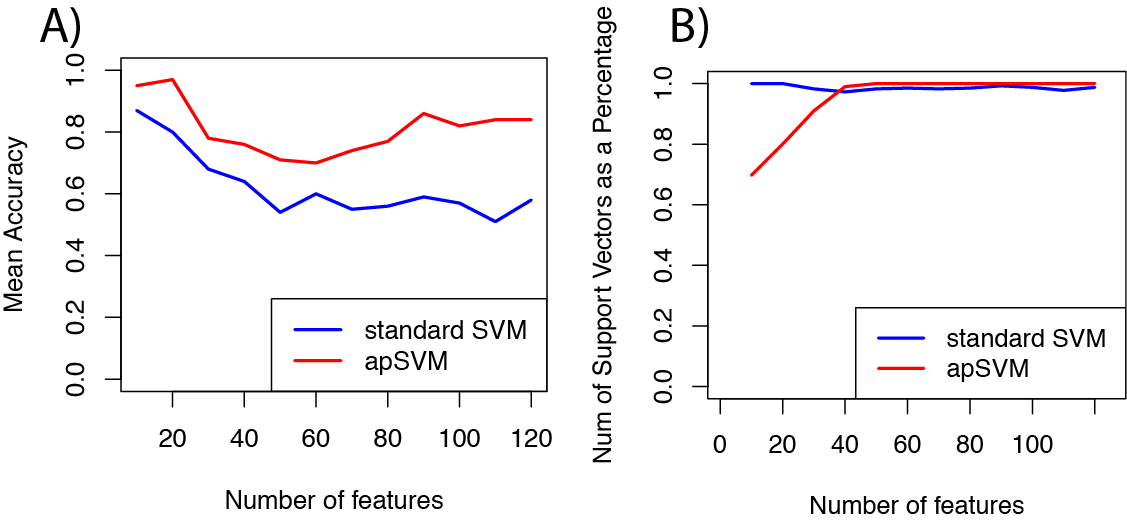}
\caption{(A) Accuracy results in simulated anomaly classification
  data. The anti-profile SVM achieves better accuracy than the
  standard SVM. (B) Stability results in simulated data. The
  anti-profile SVM uses a smaller proportion of training points as
  support vectors. SVMs that use fewer support vectors are more robust
and stable.}
\end{figure}

Figure 2a shows the accuracy of a standard SVM and the apSVM
using an RBF kernel for simulated data with $\sigma_Z=1, \sigma_{A^-}=2, \sigma_{A^+}=4$.  With a radial basis 
kernel, the anti-profile SVM was able to achieve better classification than the regular SVM.

We characterize the stability of a classifier using the proportion of
training samples that are selected as support vectors. Classifiers
that use a small proportion of points as support vectors are more
robust and stable to stochasticity in the sampling process. The more support vectors used by an SVM, the more
likely it is that the classification boundary will change with any changes in
the training data. Hence a boundary that is defined by only a few support vectors
will result in a more robust, reliable SVM. Figure 2b shows that in the simulation study the apSVM used 
fewer support vectors than the standard SVM while obtaining better accuracy. 

\section{Application to cancer genomics}
\label{section:bio}

The motivation for this work is from recent studies of epigenetic mechanisms
of cancer. Epigenetics is the study of mechanisms by which the expression level of a
gene (i.e. the degree to which a gene can exert it's influence) can be modified without
any change in the underlying DNA sequence. Recent results show that certain changes in
DNA methylation are closely associated with the occurrence of multiple
cancer types~\cite{Hansen:2011gu}. In particular, the existence of
highly-variable DNA-methylated regions in cancer as compared to normals(i.e. healthy
tissue) has been shown.  Furthermore, these highly-variable regions are
associated with tissue differentiation, and are present across
multiple cancer types. 
Another important observation made
there is that adenomas, which are benign tumors, show intermediate levels of
hyper-variability in the same DNA-methylated regions as compared to cancer and
normals.

This presents an interesting machine learning problem: 
distinguishing between cancer and adenoma based on the
hyper-variability of their methylation levels with respect to normal samples? 
A successful tool that can classify between the two groups can have far-reaching benefits in 
the area of personalized medicine and diagnostics. Since the two classes are essentially
differentiated by the degree of variability they exhibit with respect to normals, for our 
purpose we can abstract the problem to the setting we present here as
anomaly classification.

\subsection{Methylation data results}

We study the performance of the apSVM in a dataset of DNA methylation
measurements obtained for colon tissue
from 25 healthy samples, 19 adenoma samples, and 16 cancer samples,
for 384 specific positions in the human genome~\cite{Hansen:2011gu}. As mentioned previously, the cancer samples exhibit higher
variance than healthy samples, with adenoma samples showing an
intermediate level of variability (Figure 1).
We used the same classification methods mentioned
in the previous section, but with multiple runs, for each run randomly choosing 80\% of 
tumor samples for training and the remaining for testing. Figure 3 shows the
results obtained using a radial basis kernel. While the indirect kernel performs
either at the same level or marginally better than the regular kernel, then anti-profile SVM 
uses much less support vectors than the standard SVM, thus providing a much more
robust classifier.

\begin{figure}
\centering
\includegraphics{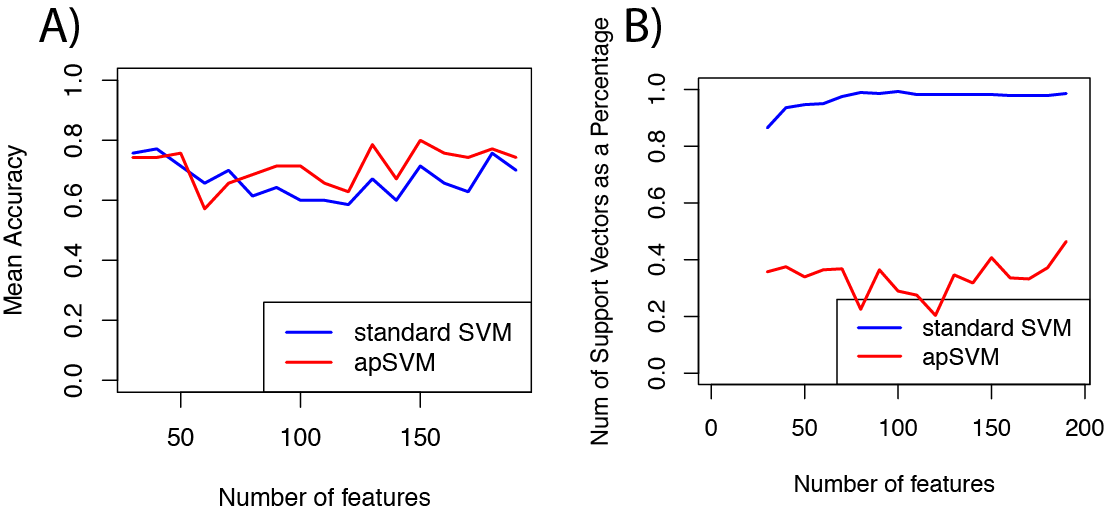}
\caption{Classification results for DNA methylation in
  cancer~\cite{Hansen:2011gu}. While both a standard SVM and the
  anti-profile SVM achieve similar accuracy using an RBF kernel (A),
  the anti-profile SVM uses much fewer support vectors.}
\end{figure}

\subsection{Expression data results}

We further applied our method to gene expression data obtained with a clinical
experiment on adrenocortical cancer~\cite{Giordano:2009bj}. The data contains expression levels for
54675 probesets, for 10 healthy samples, 22 adenoma samples, and 32 cancer
samples. The data shows the same pattern with regard to hyper-variability as the methylation 
data. Using the same methods as before, the results obtained using a linear kernel 
are shown in Figure 4. For feature selection, the
features were ranked according to $\log \frac{\mathrm{var}(Carcinoma)}{\mathrm{var}(Adenoma)}$
and for a given number $n$ as the number of features to be used, $n$ features with
the highest variance ratio were selected. While both the standard SVM
and the apSVM provided almost perfect
classification, there is a significant difference in the number of support
vectors used, with the indirect kernel requiring much fewer support vectors and
hence providing a more stable classifier.

\begin{figure}
\centering
\includegraphics{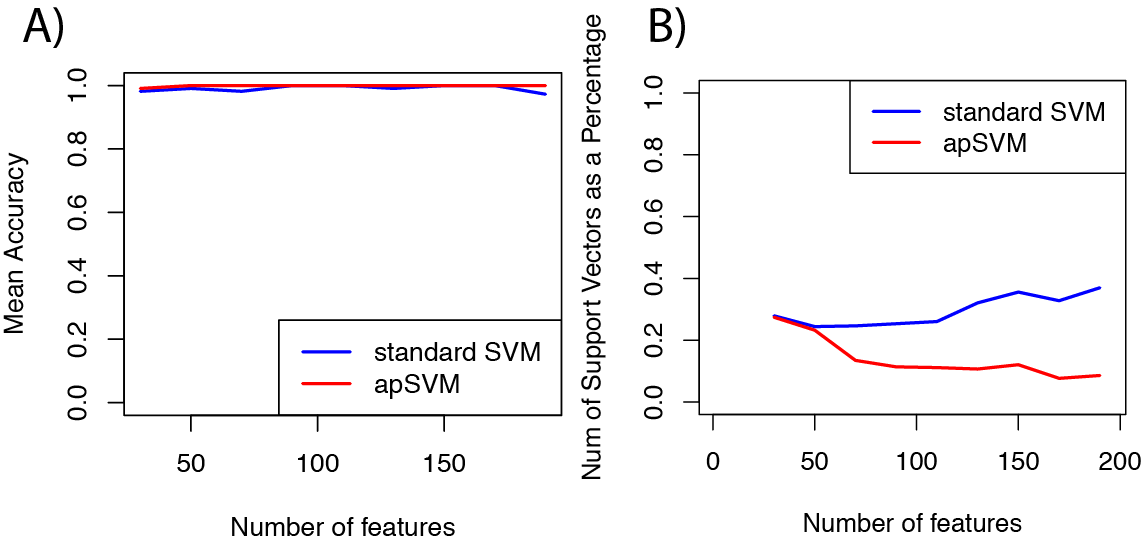}
\caption{Classification results for gene expression in
  cancer~\cite{Giordano:2009bj}. Similar to Figure 3, the accuracy of
  both the standard and anti-profile SVM is similar (in this case
  almost perfect testset accuracy is achieved by both
  classifiers). However, the anti-profile SVM again uses fewer support
vectors, leading to classifiers that are more robust and stable.}
\end{figure}

\section{Discussion}

We have introduced the anti-profile Support Vector Machine as a novel
algorithm to address the anomaly classification problem. We have shown
that under the assumption that the classes we are trying to
distinguish with a classifier are heterogeneous with respect to a
third stable class, we can define a Support Vector Machine based on an
indirect kernel using the stable class. We have shown that the dual of
the apSVM optimization problem is equivalent to that of the standard
SVM with the addition of an indirect kernel that measures similarity
of anomalous samples through similarity to the stable normal
class. Furthermore, we have characterized this indirect kernel as the
inner product in a Reproducing Kernel Hilbert Space between
representers that are projected to the subspace spanned by the
representers of the normal samples. This led to the result that the
apSVM will learn classifiers that are more robust and stable than a
standard SVM in this learning setting. We have shown by simulation and
application to cancer genomics datasets that the anti-profile SVM
does in fact produce classifiers that are more accurate and stable
than the standard SVM in this setting.

While the motivation and examples provided here are based on cancer
genomics we expect that the anomaly classification setting is
applicable to other areas. In particular, we have started looking at
the area of statistical debugging as a suitable application~\cite{Zheng:2003td}. 

The characterization of the indirect kernel through
projection to the normal subspace also suggests other possible
classifiers suitable to this task. For instance, by defining a margin
based on the projection distance directly. Furthermore, connections to
kernel methods for quantile estimation~\cite{Scholkopf:2001jp} will be
interesting to explore.

Another direction of interesting research would be to further solidify
the stability characterization we provide in
Section~\ref{sec:generalResult}. For instance, by exploring the
relationship to other leave-one-out
bounds~\cite{opper2000gpc,joachims2000egp,wahba:rbg,jaakkola1999pkr},
and the span rule for kernel quantile
estimation~\cite{scholkopf2002lk}. 

%\subsubsection*{Acknowledgments}

%\subsubsection*{References}
\bibliography{apsvm}

\begin{thebibliography}{16}
\providecommand{\natexlab}[1]{#1}
\providecommand{\url}[1]{\texttt{#1}}
\expandafter\ifx\csname urlstyle\endcsname\relax
  \providecommand{\doi}[1]{doi: #1}\else
  \providecommand{\doi}{doi: \begingroup \urlstyle{rm}\Url}\fi

\bibitem[Chandola et~al.(2009)Chandola, Banerjee, and Kumar]{Chandola}
Varun Chandola, Arindam Banerjee, and Vipin Kumar.
\newblock Anomaly detection: A survey.
\newblock \emph{ACM Comput. Surv.}, 41\penalty0 (3):\penalty0 15:1--15:58, July
  2009.
\newblock ISSN 0360-0300.
\newblock \doi{10.1145/1541880.1541882}.
\newblock URL \url{http://doi.acm.org/10.1145/1541880.1541882}.

\bibitem[Giordano et~al.(2009)Giordano, Kuick, Else, Gauger, Vinco, Bauersfeld,
  Sanders, Thomas, Doherty, and Hammer]{Giordano:2009bj}
Thomas~J Giordano, Rork Kuick, Tobias Else, Paul~G Gauger, Michelle Vinco,
  Juliane Bauersfeld, Donita Sanders, Dafydd~G Thomas, Gerard Doherty, and Gary
  Hammer.
\newblock {Molecular classification and prognostication of adrenocortical
  tumors by transcriptome profiling.}
\newblock \emph{Clinical cancer research : an official journal of the American
  Association for Cancer Research}, 15\penalty0 (2):\penalty0 668--676, January
  2009.

\bibitem[Hansen et~al.(2011)Hansen, Timp, Bravo, Sabunciyan, Langmead,
  McDonald, Wen, Wu, Liu, Diep, Briem, Zhang, Irizarry, and
  Feinberg]{Hansen:2011gu}
Kasper~Daniel Hansen, Winston Timp, H{\'e}ctor~Corrada Bravo, Sarven
  Sabunciyan, Benjamin Langmead, Oliver~G McDonald, Bo~Wen, Hao Wu, Yun Liu,
  Dinh Diep, Eirikur Briem, Kun Zhang, Rafael~A Irizarry, and Andrew~P
  Feinberg.
\newblock {Increased methylation variation in epigenetic domains across cancer
  types.}
\newblock \emph{Nature Genetics}, 43\penalty0 (8):\penalty0 768--775, August
  2011.

\bibitem[Hastie et~al.(2004)Hastie, Rosset, Tibshirani, and Zhu]{Hastie:2004uj}
Trevor Hastie, Saharon Rosset, Robert Tibshirani, and Ji~Zhu.
\newblock {The Entire Regularization Path for the Support Vector Machine}.
\newblock \emph{The Journal of Machine Learning Research}, 5:\penalty0
  1391--1415, December 2004.

\bibitem[Jaakkola and Haussler(1999)]{jaakkola1999pkr}
T.~Jaakkola and D.~Haussler.
\newblock {Probabilistic kernel regression models}.
\newblock \emph{Proceedings of the 1999 Conference on AI and Statistics}, 1999.

\bibitem[Joachims(2000)]{joachims2000egp}
T.~Joachims.
\newblock {Estimating the generalization performance of a SVM efficiently}.
\newblock \emph{Proceedings of the International Conference on Machine
  Learning}, 2000.

\bibitem[Karatzoglou et~al.(2004)Karatzoglou, Smola, Hornik, and
  Zeileis]{kernlab}
Alexandros Karatzoglou, Alex Smola, Kurt Hornik, and Achim Zeileis.
\newblock kernlab -- an {S4} package for kernel methods in {R}.
\newblock \emph{Journal of Statistical Software}, 11\penalty0 (9):\penalty0
  1--20, 2004.
\newblock URL \url{http://www.jstatsoft.org/v11/i09/}.

\bibitem[Manevitz and Yousef(2002)]{Manevitz}
Larry~M. Manevitz and Malik Yousef.
\newblock One-class svms for document classification.
\newblock \emph{J. Mach. Learn. Res.}, 2:\penalty0 139--154, March 2002.
\newblock ISSN 1532-4435.
\newblock URL \url{http://dl.acm.org/citation.cfm?id=944790.944808}.

\bibitem[Opper and Winther(2000)]{opper2000gpc}
M.~Opper and O.~Winther.
\newblock {Gaussian Processes for Classification: Mean-Field Algorithms}.
\newblock \emph{Neural Computation}, 12:\penalty0 2655--2684, 2000.

\bibitem[Scholkopf and Smola(2002)]{scholkopf2002lk}
B.~Scholkopf and A.J. Smola.
\newblock \emph{{Learning with Kernels}}.
\newblock MIT Press Cambridge, Mass, 2002.

\bibitem[Sch{\"o}lkopf et~al.(2001)Sch{\"o}lkopf, Platt, Shawe-Taylor, Smola,
  and Williamson]{Scholkopf:2001jp}
B~Sch{\"o}lkopf, J~C Platt, J~Shawe-Taylor, A~J Smola, and R~C Williamson.
\newblock {Estimating the support of a high-dimensional distribution.}
\newblock \emph{Neural computation}, 13\penalty0 (7):\penalty0 1443--1471, July
  2001.

\bibitem[Steinwart et~al.(2005)Steinwart, Hush, and
  Scovel]{Steinwart:2005:CFA:1046920.1058109}
Ingo Steinwart, Don Hush, and Clint Scovel.
\newblock A classification framework for anomaly detection.
\newblock \emph{J. Mach. Learn. Res.}, 6:\penalty0 211--232, December 2005.
\newblock ISSN 1532-4435.
\newblock URL \url{http://dl.acm.org/citation.cfm?id=1046920.1058109}.

\bibitem[Vapnik and Chapelle(2000)]{vapnik2000bee}
V.~Vapnik and O.~Chapelle.
\newblock {Bounds on Error Expectation for Support Vector Machines}.
\newblock \emph{Neural Computation}, 12:\penalty0 2013--2036, 2000.

\bibitem[Wahba(1999)]{wahba1999svm}
G.~Wahba.
\newblock {Support vector machines, reproducing kernel Hilbert spaces, and
  randomized GACV}.
\newblock \emph{Advances in kernel methods: support vector learning}, pages
  69--88, 1999.

\bibitem[Wahba et~al.(2001)Wahba, Lin, Lee, and Zhang]{wahba:rbg}
G.~Wahba, Y.~Lin, Y.~Lee, and H.~Zhang.
\newblock {On the relation between the GACV and Joachims' $\xi$$\alpha$ method
  for tuning support vector machines, with extensions to the nonstandard case}.
\newblock Technical report, Technical Report 1039, Statistics Department
  University of Wisconsin, Madison WI, 2001, 2001.

\bibitem[Zheng et~al.(2003)Zheng, Jordan, and Liblit]{Zheng:2003td}
AX~Zheng, MI~Jordan, and B~Liblit.
\newblock {Statistical debugging of sampled programs}.
\newblock \emph{Advances in Neural Information Processing Systems}, 16, 2003.

\end{thebibliography}

\end{document}